\documentclass[english,letterpaper, 10 pt, conference]{ieeeconf}
\usepackage[T1]{fontenc}
\usepackage[latin9]{inputenc}
\usepackage{verbatim}
\usepackage{amsmath}
\usepackage{amssymb}
\usepackage{graphicx}

\makeatletter
\newtheorem{thm}{\protect\theoremname}

\IEEEoverridecommandlockouts   

\makeatother

\usepackage{babel}
\providecommand{\theoremname}{Theorem}

\begin{document}

\title{\textbf{Tilt estimator for 3D non-rigid pendulum based on a tri-axial
accelerometer and gyrometer}}

\author{Mehdi Benallegue, Abdelaziz Benallegue, Yacine Chitour \thanks{ M. Benallegue is with Humanoid Research Group, National Institute of Advanced Industrial Science and Technology (AIST), Tsukuba, Ibaraki, Japan. A. Benallegue is with JRL-AIST (Joint Robotics Laboratory), Tsukuba, Ibaraki, Japan and Laboratoire d'Ing\'enierie des Syst\`emes de Versailles, France and Y. Chitour is with Universit\'e Paris-Sud, CentraleSup\'elec, CNRS, France.       {\tt\small benalleg@lisv.uvsq.fr, mehdi.benallegue@aist.go.jp, yacine.chitour@lss.supelec.fr} } 
}
\maketitle
\begin{abstract}
The paper presents a new observer for tilt estimation of a 3-D non-rigid
pendulum. The system can be seen as a multibody robot attached to
the environment with a ball joint for which there is no sensor. The
estimation of tilt, i.e. roll and pitch angles, is mandatory for balance
control for a humanoid robot and all tasks requiring verticality.
Our method obtains tilt estimations using joints encoders and inertial
measurements given by an IMU equipped with tri-axial accelerometer
and gyrometer mounted in any body of the robot. The estimator takes
profit from the kinematic coupling resulting from the pivot constraint
and uses the entire signal of accelerometer including linear accelerations.
Almost Global Asymptotic convergence of the estimation errors is proven
together with local exponential stability. The performance of the
proposed observer is illustrated by simulations. 
\end{abstract}

\section{Introduction}

One predominant goal of robotics is to be able to perform versatile
interactions with the environment. In some cases, contact point constitute
a link between the floating base of the robot and the environment,
one example is legged locomotion, but also environment-related tasks
such as torquing or drilling. Most of these tasks require the contact
point to remain at a precise position and not to detach or slip. The
observance of such a constraint generates a kinematic coupling allowing
to model the robot as a kinematic chain attached to the environment
with an unactuated joint. This can be simply summarized as a pendulum
with the contact as the pivot point. 

One main issue regarding this class of systems is that beside the
unactuation, there is usually no direct measurement of the configuration
of this pivot. Of course, properly estimating this configuration is
of crucial importance in most tasks. Nevertheless, several kinds of
sensors are sensitive to this configuration, and may be used to estimate
it. The most broadly used ones are tri-axial accelerometer and gyrometer.
This set of sensors provides invariant signals relative to different
rotations around the gravitational field direction. This means that
this orientation, usually called yaw angle, is not observable using
this sensing system~\cite{benallegue2014humanoids}.

\begin{figure}
\begin{centering}
\includegraphics[width=0.6\columnwidth]{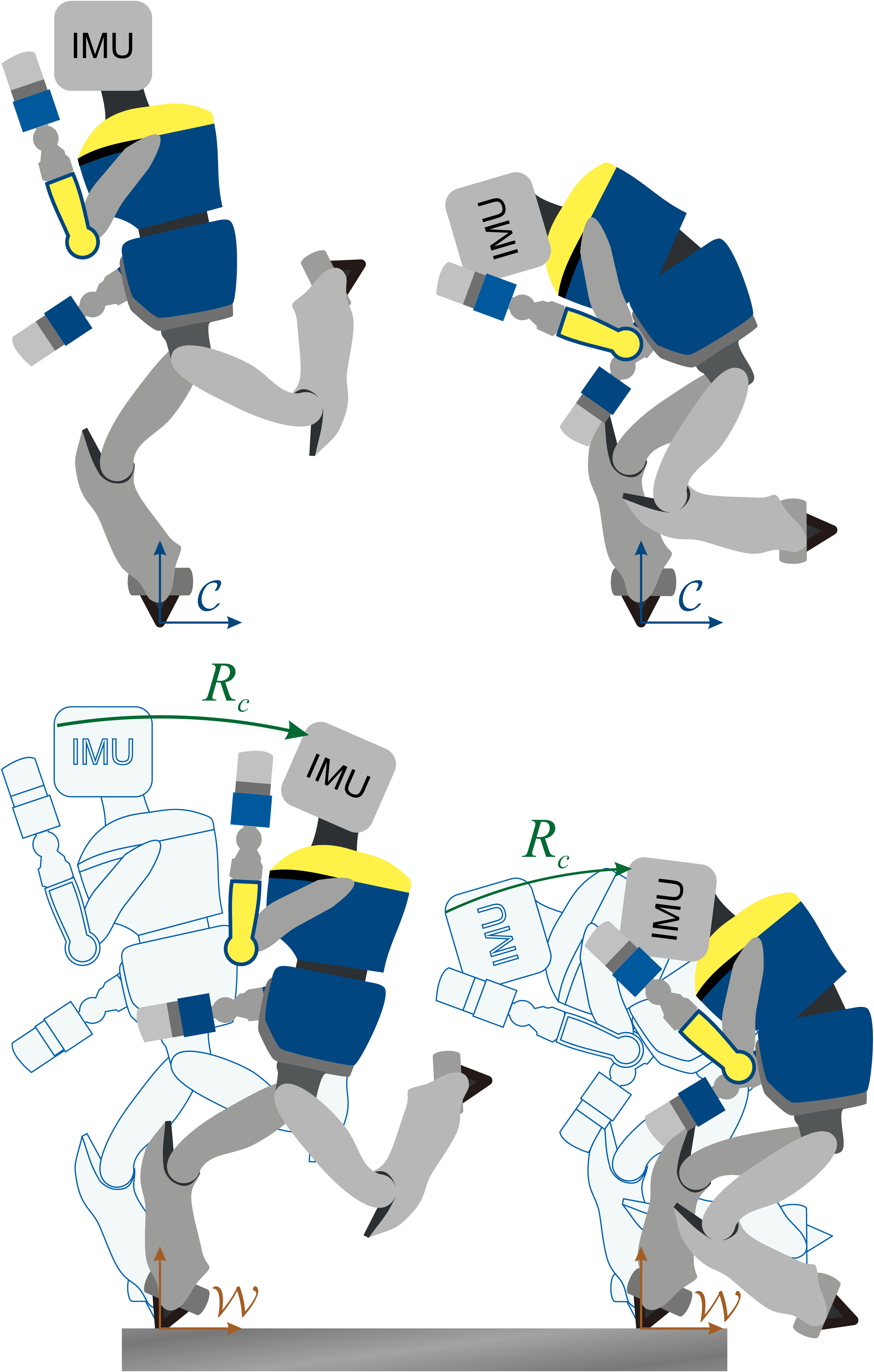}
\par\end{centering}
\caption{Top figures: the robot is attached to the environment through a 3D
pivot joint. The encoders only provide the configuration in a local
frame called control frame and represented by $\mathcal{C}$. Bottom
figures: In the world frame $\mathcal{W}$ the configuration of the
real robot is rotated compared to its value in the control frame (light
wireframe robot). Only the IMU can provide data about this rotation,
$R_{c}$. We need to account for encoders in this estimation, to distinguish
local kinematics from pivot positions. For example the rotation $R_{c}$
is here identical in the right and the left lower figures with different
IMU configurations.\label{fig:sketch}}
\end{figure}

Nevertheless, in robotics there is often a specific need for a precise
estimation of the two other degrees of freedom, which can be referred
to as roll and pitch angles, or simply tilt. These two degrees of
freedom describe the configuration of the pivot relative to the gravitational
field. They are then essential for maintaining balance for humanoid
robots \cite{Wieber2016}.

In this paper we provide a state estimator which aims at addressing
this problem by providing a state estimator for the tilt of a pendulum
which uses an accelerometer and a gyrometer, (i) without neglecting
linear accelerations compared to gravity, (ii) well suited for articulated
robot, and (iii) with a proven Lyapunov stability. Furthermore, this
estimator reaches local exponential stability performances, which
makes it particularly suitable for the use as a state feedback for
closed-loop control. The idea behind this estimator is close to the
recent development of Hua \emph{et al }\cite{Hua2016automatica},
who propose a tilt estimation when having a velocity measurement.
The gyrometer and the kinematics constitute the velocity measurement
in our case. But the estimator we develop is simpler and has better
convergence properties in both theoretical and simulation points of
view.

The section~\ref{sec:Problem-statement} presents the issue treated
in this paper together with the model of the system and the sensors.
The section~\ref{sec:State-estimator} presents the development of
the state estimator. The section \ref{sec:Stability-analysis} analyzes
the stability of the estimation error. Section \ref{sec:Simulations}
shows the performances of the estimator in simulation, and finally
the section \ref{sec:Discussion-and-conclusion} discusses the results
and the properties of this estimator.

\section{Problem statement\label{sec:Problem-statement}}

The system is a robot linked to the environment through a ball joint
called pivot. Without loss of generality we consider that the pivot
is located at the origin of the inertial world frame ($\mathcal{W}$).
The configuration of the pivot is a pure 3D rotation describing a
transformation between the global frame and the local frame of the
robot, also called control frame ($\mathcal{C}$). We represent this
rotation by the rotation matrix $R_{c}\in\mathbb{R}^{3\times3}$.
For instance, the sensor $s$ located at position $^{c}p_{s}\in\mathbb{R}^{3}$
and orientation $^{c}R_{s}\in\mathbb{R}^{3\times3}$ in $\mathcal{C}$
is actually at $p_{s}=R_{c}{}^{c}p_{s}$ and has the orientation $R_{s}=R_{c}{}^{c}R_{s}$
in $\mathcal{W}$. This problem is sketched in Figure~\ref{fig:sketch}.

There is no sensor providing the pivot configuration. Instead, the
robot is equipped with an IMU consisting in an accelerometer and a
gyrometer, both of them are on three axes. The position of this IMU
may be not rigidly linked to the ball joint, and can be located in
another body of the robot. Since the robot can modify its actuated
joint kinematics the IMU may move in the control frame. Therefore,
we have to consider its position $\,^{c}p_{s}\in\mathbb{R}^{3}$,
its orientation represented by the orthogonal matrix $\,^{c}R_{s}\in\mathbb{R}^{3\times3}$,
together with their respective first-order time-derivatives $\,^{c}\dot{p}_{s}\in\mathbb{R}^{3}$
and $\,^{c}\omega_{s}\in\mathbb{R}^{3}$ such that $\,^{c}\dot{R}_{s}=S(^{c}\omega_{s})\,^{c}R_{s}$,
where $S$ is the skew-symmetric operator. %

The values of $\,^{c}\dot{p}_{s}$, $\,^{c}R_{s}^{T}$ , $\,^{c}\dot{p}_{s}$
and $\,^{c}\omega_{s}$ can be obtained through the positions and
velocities of the joint encoders and are often the outcome of a motion
controller. Therefore, these values are considered to be perfectly
known.

The accelerometer provides the sum of the gravitational field and
the linear acceleration of the sensor, expressed in the sensor frame.
In other words we have 
\begin{equation}
y_{a}=\,R_{s}^{T}\left(g_{0}e_{z}+\ddot{p}_{s}\right),\label{eq:measure-accelero-0}
\end{equation}
 where $y_{a}$, $p_{s}$, $R_{s}$, $g_{0}$ and $e_{z}$ are respectively
the accelerometer measurements, the position and the orientation of
the IMU, standard gravity constant and a unit vector such that $-g_{0}e_{z}$
is the gravitational field.

The gyrometer provides the angular velocity of the IMU, expressed
in the sensor frame. In other words 
\begin{equation}
y_{g}=\,R_{s}^{T}\omega_{s},\label{eq:measure-gyro0}
\end{equation}
 where $\omega_{s}$ is the angular velocity vector of the sensor
in the global frame such that $\dot{R}_{s}=S(\omega_{s})\,R_{s}$. 

We can see from these equations that the measurements are invariant
regarding rotations around the vector $e_{z}$ aligned with the gravitational
field. Therefore, the orientation that can be estimated through this
sensing system is incomplete. Nevertheless, we show here that one
partial information is observable and consists in $R_{c}^{T}e_{z}$,
the direction of the gravitational field in the local frame of the
robot. This data is the most important variable required to control
balance and may be considered as a measure of ``verticality'' in general.

By replacing $R_{s}$ and $p_{s}$ by $R_{c}{}^{c}R_{s}$ and $R_{c}{}^{c}p_{s}$
respectively and performing time-derivations and identification with
(\ref{eq:measure-gyro0}) and (\ref{eq:measure-accelero-0}) obtain
\begin{align}
y_{g}= & \,^{c}R_{s}^{T}\,^{c}\omega_{s}+\,^{c}R_{s}^{T}R_{c}^{T}\omega_{c},\label{eq:measure-gyro1}\\
y_{a}= & \,^{c}R_{s}^{T}\left(\left(S(R_{c}^{T}\dot{\omega}_{c})+S^{2}(R_{c}^{T}\omega_{c})\right)\,^{c}p_{s}+2S(R_{c}^{T}\omega_{c})\,^{c}\dot{p}_{s}\right)\nonumber \\
 & +\,^{c}R_{s}^{T}{}^{c}\ddot{p}_{s}+g_{0}\,^{c}R_{s}^{T}R_{c}^{T}e_{z}.\label{eq:measure-accelero-1}
\end{align}
where $\omega_{c}$ is the angular velocity vector of the pendulum
such that $\dot{R}_{c}=S(\omega_{c})\,R_{c}$.

In the following section we develop the state observer for the estimation
of $R_{c}^{T}e_{z}$.

\section{State estimator\label{sec:State-estimator}}

\subsection{State definition}

The first variable we define is the pivot angular velocity expressed
in the control frame $y_{1}=R_{c}^{T}\omega_{c}$. Replacing this
in (\ref{eq:measure-gyro1}) we have 
\begin{align}
y_{1}= & ^{c}R_{s}\left(y_{g}-{}^{c}R_{s}^{T}\,^{c}\omega_{s}\right),\label{eq:y1expression}
\end{align}
 and since all the rightmost variables are known or measured we may
consider $y_{1}$ as measured.

Let's define also the following state variables:
\begin{eqnarray}
x_{1} & = & S(^{c}p_{s})y_{1}-\,^{c}\dot{p}_{s},\\
x_{2} & = & R_{c}^{T}e_{z},
\end{eqnarray}
with $x_{1}\in\mathbb{R}^{3}$ and $x_{2}\in\mathbb{S}^{2}$, with
the set $\mathbb{S}^{2}\subset\mathbb{R}^{3}$ is the unit sphere
centered at the origin, and defined as 
\[
\mathbb{S}^{2}=\left\{ x\in\mathbb{R}^{3}/\left\Vert x\right\Vert =1\right\} .
\]

The variable $x_{2}$ is the state we aim at estimating and cannot
be obtained algebraically. On the contrary, the variable $x_{1}$
is considered measured since we know $^{c}p_{s}$ , $^{c}\dot{p}_{s}$
and $y_{1}$, and is the opposite of the linear velocity of the IMU
expressed in the local frame $x_{1}=-R_{c}^{T}\dot{p}_{s}$. In this
study we use this data to build a tilt estimator able to distinguish
gravity from accelerations, similarly to\emph{ }\cite{Hua2016automatica}\emph{.} 

By left-multiplying Equation (\ref{eq:measure-accelero-1}) by $^{c}R_{s}$
and replacing the expression of $y_{1}$ of equation (\ref{eq:y1expression})
we get
\begin{align}
 & S(^{c}p_{s})R_{c}^{T}\dot{\omega}_{c}+S(\,^{c}\dot{p}_{s})y_{1}-{}^{c}\ddot{p}_{s}=\nonumber \\
 & \qquad\;\;\;-S(y_{1})\left(S(\,^{c}p_{s})y_{1}-\,^{c}\dot{p}_{s}\right)+g_{0}R_{c}^{T}e_{z}-{}^{c}R_{s}y_{a}.\label{eq:6}
\end{align}

We notice that the left member of equation (\ref{eq:6}) is the first
order time-derivative of $x_{1}$. This, together with the time-differentiation
of $x_{2}$, provide us with the following dynamic equations 
\begin{equation}
\begin{cases}
\dot{x}_{1} & =-S(y_{1})x_{1}+g_{0}x_{2}-{}^{c}R_{s}y_{a},\\
\dot{x}_{2} & =-S(y_{1})x_{2}.
\end{cases}\label{eq:dynamics}
\end{equation}

The system ((\ref{eq:dynamics})) is suitable for the observer synthesis.

\subsection{State-observer and error dynamics:}

In order to estimate $x_{2}=R_{c}^{T}e_{z}$, we propose the following
state-observer
\begin{equation}
\begin{cases}
\dot{\hat{x}}_{1} & =-S(y_{1})\hat{x}_{1}+g_{0}\hat{x}_{2}-{}^{c}R_{s}y_{a}+\alpha(x_{1}-\hat{x}_{1}),\\
\dot{\hat{x}}_{2} & =-S(y_{1}-\beta S(\hat{x}_{2})(x_{1}-\hat{x}_{1}))\hat{x}_{2},
\end{cases}\label{eq:observer}
\end{equation}
where $\alpha$, $\beta$ are positive scalar gains which verify the
condition $\beta g_{0}<\alpha^{2}$ and $\hat{x}_{1}$ and $\hat{x}_{2}$
are the estimations of $x_{1}$ and $x_{2}$ respectively. 

The initial value of $\hat{x}_{2}$ should be in $\mathbb{S}^{2}$.
Then the dynamics of the last equation ensures that the norm of this
vector remains constant in time. The initial value for $\hat{x}_{1}$
on its side could be anywhere in $\mathbb{R}^{3}$.

We define the following estimation errors $\tilde{x}_{1}=x_{1}-\hat{x}_{1}$
and $\tilde{x}_{2}=x_{2}-\hat{x}_{2}$, a time-differentiation of
these expressions provide us with the following error dynamics:
\begin{equation}
\begin{cases}
\dot{\tilde{x}}_{1} & =-S(y_{1})\tilde{x}_{1}-\alpha\tilde{x}_{1}+g_{0}\tilde{x}_{2},\\
\dot{\tilde{x}}_{2} & =-S(y_{1})\tilde{x}_{2}+\beta S^{2}(\hat{x}_{2})\tilde{x}_{1}.
\end{cases}\label{eq:error_dynamics}
\end{equation}

To run the analysis of errors, we set $z_{i}=R_{c}\tilde{x}_{i}$.
We notice also that $\dot{R}_{c}=R_{c}S(R_{c}^{T}\omega_{c})=R_{c}S(y_{1})$
and $R_{c}(\tilde{x}_{2}+\hat{x}_{2})=e_{z}$, we obtain this new
error dynamics 
\begin{equation}
\begin{cases}
\dot{z}_{1} & =-\alpha z_{1}+g_{0}z_{2},\\
\dot{z}_{2} & =\beta S^{2}(e_{z}-z_{2})z_{1}.
\end{cases}\label{eq:error_dynamics1-1}
\end{equation}

The nice property of this new dynamics is that it is autonomous and
defines a time-invariant ordinary differential equation (ODE) which
simplifies drastically the stability analysis. In fact, if one define
the state $\xi:=\left(z_{1},z_{2}\right)$ and the state space $\varUpsilon:=\mathbb{R}^{3}\times\mathbb{S}_{e_{z}}$
with $\mathbb{S}_{e_{z}}=\left\{ z\in\mathbb{R}^{3}|\left(e_{z}-z\right)\in\mathbb{S}^{2}\right\} $,
one can write (\ref{eq:error_dynamics1-1}) as $\dot{\xi}=F\left(\xi\right)$
where $F$ gathers the right-hand side of (\ref{eq:error_dynamics1-1})
and defines a smooth vector field on $\varUpsilon$.

\section{Stability analysis\label{sec:Stability-analysis}}

\subsection{Asymptotic stability}

Let's consider the following positive-definite differentiable function
$V:\varUpsilon\rightarrow\mathbb{R}^{+}$ 
\begin{eqnarray}
V & = & \frac{\left\Vert \alpha z_{1}-g_{0}z_{2}\right\Vert ^{2}}{2}+g_{0}^{2}\frac{\left\Vert z_{2}\right\Vert ^{2}}{2},\label{eq:Lyap-1}
\end{eqnarray}
which is radially unbounded over $\varUpsilon$. 
\begin{thm}
\emph{The time-invariant ODE defined by (\ref{eq:error_dynamics1-1})
verifies the following }
\begin{enumerate}
\item \emph{It admits two equilibrium points namely the origin $(0,0)$
and $(\frac{2g_{0}}{\alpha}e_{z},2e_{z})$.}
\item \emph{All trajectories of (\ref{eq:error_dynamics1-1}) converge to
one of the equilibrium points defined in item 1.}
\item \emph{The equilibrium ($0,0$) is locally asymptotically stable with
a domain of attraction containing the set 
\begin{equation}
V_{c}:=\left\{ \xi=\left(z_{1},z_{2}\right)\in\varUpsilon\mid V\left(\xi\right)<2g_{0}^{2}\right\} .
\end{equation}
}
\item \emph{The system (\ref{eq:error_dynamics1-1}) is almost globally
stable with respect to the origin in the following sense: there exists
an open dense subset $\varUpsilon_{0}\subset\varUpsilon$ such that,
for every initial condition $\xi_{0}\in\varUpsilon_{0}$, the corresponding
trajectory converges asymptotically to ($0,0$).}
\end{enumerate}
\end{thm}
\begin{proof}
Let's prove the four items of the theorem
\begin{enumerate}
\item The equilibria are calculated by solving the equation $F\left(\xi\right)=0$,
where $F$ is the nonlinear function describing (\ref{eq:error_dynamics1-1}),
we get the following
\begin{equation}
\begin{cases}
0 & =-\alpha z_{1}+g_{0}z_{2},\\
0 & =\beta S^{2}(e_{z}-z_{2})z_{1}.
\end{cases}
\end{equation}
The trivial solution is $(0,0)$ and the second solution is calculated
if we consider that $\left(z_{1},z_{2}\right)\neq\left(0,0\right)$,
so we can write
\begin{eqnarray}
z_{1} & = & \frac{g_{0}}{\alpha}z_{2},\label{equilibrium_y1}\\
0 & = & \beta\frac{g_{0}}{\alpha}S(e_{z}-z_{2})S(e_{z})z_{2}.\label{equilibirum_y2}
\end{eqnarray}
We know that $z_{2}\in\mathbb{S}_{e_{z}}$, so the only solution of
(\ref{equilibirum_y2}) is $z_{2}=2e_{z}$, which gives from (\ref{equilibrium_y1})
that $z_{1}=\frac{2g_{0}}{\alpha}e_{z}$. This completes the proof
of item 1.
\item The time derivative of (\ref{eq:Lyap-1}) in view of (\ref{eq:error_dynamics1-1})
yields {\footnotesize{}
\begin{align}
\dot{V}= & -\frac{g_{0}\beta}{\alpha}\left(\alpha z_{1}-g_{0}z_{2}\right)^{T}S^{2}(e_{z}-z_{2})\left(\alpha z_{1}-g_{0}z_{2}\right)\nonumber \\
 & -\alpha\left\Vert \alpha z_{1}-g_{0}z_{2}\right\Vert ^{2}+\frac{g_{0}^{3}\beta}{\alpha}z_{2}^{T}S^{2}(e_{z}-z_{2})z_{2}\nonumber \\
= & -\alpha\left(1-G_{0}\right)\left\Vert \alpha z_{1}-g_{0}z_{2}\right\Vert ^{2}+\alpha g_{0}^{2}G_{0}z_{2}^{T}S^{2}(e_{z})z_{2}\nonumber \\
 & \qquad-\alpha G_{0}\left(\left(\alpha z_{1}-g_{0}z_{2}\right)^{T}(e_{z}-z_{2})\right)^{2}.\label{eq:D_Lyap-2}
\end{align}
}where $G_{0}=\frac{\beta g_{0}}{\alpha^{2}}$, one easily verifies
that within the gain condition $\beta g_{0}<\alpha^{2}$ we have $\dot{V}<0$
if $(z_{1},z_{2})$ is not an equilibrium. Since (\ref{eq:error_dynamics1-1})
is autonomous and $V$ is radially unbounded, one can use LaSalle's
invariance theorem. Therefore, every trajectory converges to a trajectory
along which $\dot{V}\equiv0$.
\item Since $V$ is non-increasing, $V\left(\xi\right)<2g_{0}^{2}$ at $t=0$,
implies that $\left\Vert z_{2}(t)\right\Vert <2$ for every $t\geq0$.
Since the trajectory converges to one of the two equilibrium points,
it must be ($0,0$) because this is the only one contained in $V_{c}$.
\item The linearized system around the equilibrium $(\frac{2g_{0}}{\alpha}e_{z},2e_{z})$
is given by the following dynamics 
\begin{equation}
\dot{X}=AX,
\end{equation}
with $X=\left(\begin{array}{cc}
\left(z_{1}-\frac{2g_{0}}{\alpha}e_{z}\right)^{T} & \left(z_{2}-2e_{z}\right)^{T}\end{array}\right)^{T}$ and $A$ is a constant matrix having the form
\begin{equation}
A=\left[\begin{array}{ccc}
-\alpha I &  & g_{0}I\\
\\
\beta S^{2}(e_{z}) &  & -2\alpha G_{0}S^{2}(e_{z})
\end{array}\right]
\end{equation}
\end{enumerate}
The characteristic polynomial of the matrix $A$ is given by
\begin{equation}
P(\lambda)=\lambda\left(\lambda+\alpha\right)\left(\lambda^{2}+\alpha\left(1-2G_{0}\right)\lambda-g_{0}\beta\right)^{2}.
\end{equation}

We find that this polynomial has at least one positive root, which
is given by
\begin{equation}
\lambda=\alpha\frac{\sqrt{\left(1+4G_{0}^{2}\right)}-\left(1-2G_{0}\right)}{2}>0,
\end{equation}
which means that the equilibrium $(\frac{2g_{0}}{\alpha}e_{z},2e_{z})$
is unstable. This completes the proof of the theorem.
\end{proof}

\subsection{Local exponential convergence\label{subsec:Local-exponential-convergence}}

From equation (\ref{eq:D_Lyap-2}) we can write the following

\begin{equation}
\dot{V}\leq-\alpha\left(1-G_{0}\right)\left\Vert \alpha z_{1}-g_{0}z_{2}\right\Vert ^{2}+\alpha g_{0}^{2}G_{0}z_{2}^{T}S^{2}(e_{z})z_{2}
\end{equation}

In order to find the conditions of exponential convergence, let's
observe the following relations 

\begin{equation}
z_{2}^{T}S^{2}(e_{z})z_{2}=-\left\Vert z_{2}\right\Vert ^{2}+\frac{1}{4}\left(\left\Vert z_{2}\right\Vert ^{2}\right)^{2}
\end{equation}
\begin{align}
\dot{V} & \leq-\alpha\left(1-G_{0}\right)\left\Vert \alpha z_{1}-g_{0}z_{2}\right\Vert ^{2}\\
 & \qquad-\alpha g_{0}^{2}G_{0}\left(1-\frac{1}{4}\left\Vert z_{2}\right\Vert ^{2}\right)\left\Vert z_{2}\right\Vert ^{2}\nonumber 
\end{align}

In the case of $\left\Vert V(\xi)\right\Vert <2g_{0}^{2}$ at $t=0$,
we can say it exists a fixed $\epsilon>0$ such that $1-\frac{1}{4}\left\Vert z_{2}(t)\right\Vert ^{2}>\epsilon$,
since the equilibrium which correspond to $\left\Vert z_{2}\right\Vert =2$
is non attractive, so we can write the following
\begin{equation}
\dot{V}\leq-\alpha\left(1-G_{0}\right)\left\Vert \alpha z_{1}-g_{0}z_{2}\right\Vert ^{2}-\alpha g_{0}^{2}G_{0}\epsilon\left\Vert z_{2}\right\Vert ^{2}
\end{equation}
\begin{equation}
\dot{V}\leq-min(\left(1-G_{0},G_{0}\epsilon\right)\alpha\left(\left\Vert \alpha z_{1}-g_{0}z_{2}\right\Vert ^{2}+g_{0}^{2}\left\Vert z_{2}\right\Vert ^{2}\right)
\end{equation}
which can be written as
\begin{equation}
\dot{V}\leq-2min\left(1-G_{0},G_{0}\epsilon\right)\alpha V
\end{equation}
which gives the following inequality
\begin{eqnarray}
V(t) & \leq V(0)e^{-2min\left(1-G_{0},G_{0}\epsilon\right)\alpha t}
\end{eqnarray}

This leads to the local exponential convergence of the errors to the
equilibrium $(0,0)$.

\section{Simulations\label{sec:Simulations}}

In this section, we present simulation results showing the effectiveness
of the proposed estimator. We generated the signal $^{c}\omega_{s}$
with trigonometric functions and generated the trajectory of $^{c}R_{s}$
by integration. Figure \ref{fig:controlframe} shows time plot of
$^{c}R_{s}$ represented by roll, pitch and yaw angles. We generated
the trajectory of $^{c}p_{s}$ by integrating the signal $^{c}\dot{p}_{s}$
which is the sum of filtered noise and a linear feedback loop to maintain
$^{c}p_{s}$ around the value $\left(0,0,1.3\right)$. Finally we
generated $\dot{\omega}_{c}$ signals using trigonometric functions
and obtained $\omega_{c}$ and $R_{c}$ trajectories by integration.
Afterwards we generated the measurement signals $y_{a}$ for accelerometer
and $y_{g}$ for gyrometer using equations (\ref{eq:measure-gyro1})
and (\ref{eq:measure-accelero-1}).

We have considered for the simulations the initial conditions for
the estimator which correspond to the initial errors $\tilde{x}_{1}(0)=0$
and $\tilde{x}_{2}(0)=\left(\begin{array}{ccc}
-1.87 & 0.28 & 0.39\end{array}\right)^{T}$. The parameters of the estimator have been chosen as $\alpha=19.8$
and $\beta=10$, so the condition ($G_{0}=\frac{g_{0}\beta}{\alpha^{2}}<1$)
is verified. We performed two simulation tests, one without considering
noise and one with white centered Gaussian noise with standard deviation
of $0.04$ (normalized) added to the three elements of vector measurements
$y_{g}$ and with standard deviation of $0.2$ (normalized) added
to the three elements of vector measurements $y_{a}$. 

Figure \ref{fig:x1tilde} on top and bottom shows the evolution of
the estimation errors $\tilde{x}_{1}$ without noise and with noise,
respectively. Figure~\ref{fig:x2} and Figure~\ref{fig:x2noise}
show the estimation tracking of the variable $x_{2}$ and the estimation
errors $\tilde{x}_{2}$ with respect to time, without and with noise
respectively. We can see that the estimation error converges to zero
in about one second. For the noisy case, even if the estimation error
$\tilde{x}_{1}$ shows some sensitivity, we see that the error $\tilde{x}_{2}$
filters this noise in a relatively efficient way. These two figures,
\ref{fig:x2} and~\ref{fig:x2noise} compare also our results with
the observer of Hua et al\emph{ }\cite{Hua2016automatica} with equivalent
gains ($K_{1}^{v}=\alpha=19.80\text{, }K_{2}^{r}=\alpha=19.80\text{, and }K_{1}^{r}=\beta=10$,
with $K_{1}^{v}\text{, }K_{2}^{r}\text{, and }K_{1}^{r}$ parameters
of their estimator), labeled as `comparison'. We see clearly that
the performances of that estimator are not as good, in both clean
and noisy cases, as the presented one, with twice longer convergence
times.

\begin{figure}
\includegraphics[width=1\columnwidth]{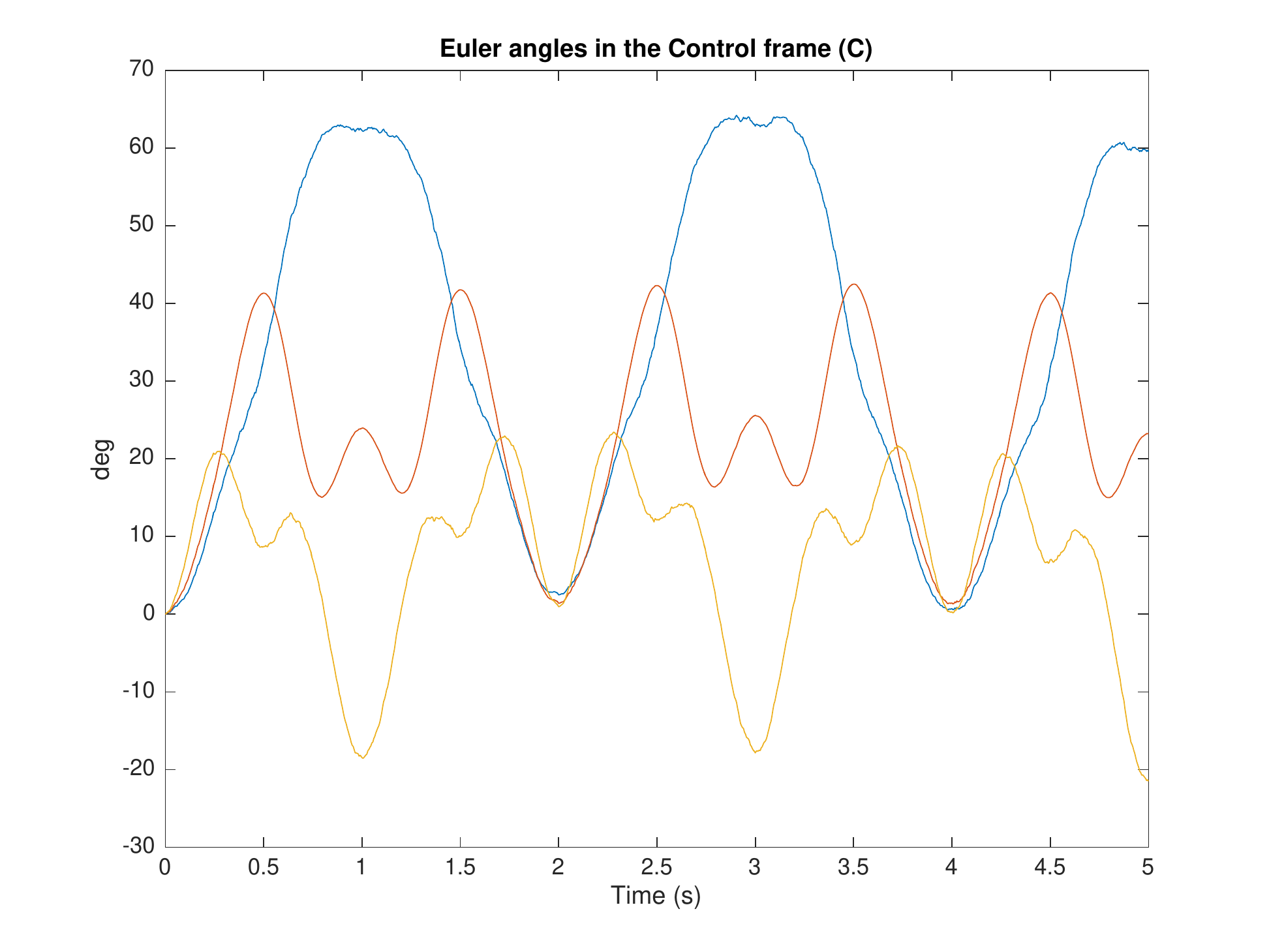}

\caption{Plot showing the values of the orientation of the IMU in the control
frame $\left(\mathcal{C}\right)$ expressed using Euler angles (blue:
roll, red: pitch, orange: yaw)\label{fig:controlframe}}
\end{figure}

\begin{figure}
\includegraphics[width=1\columnwidth]{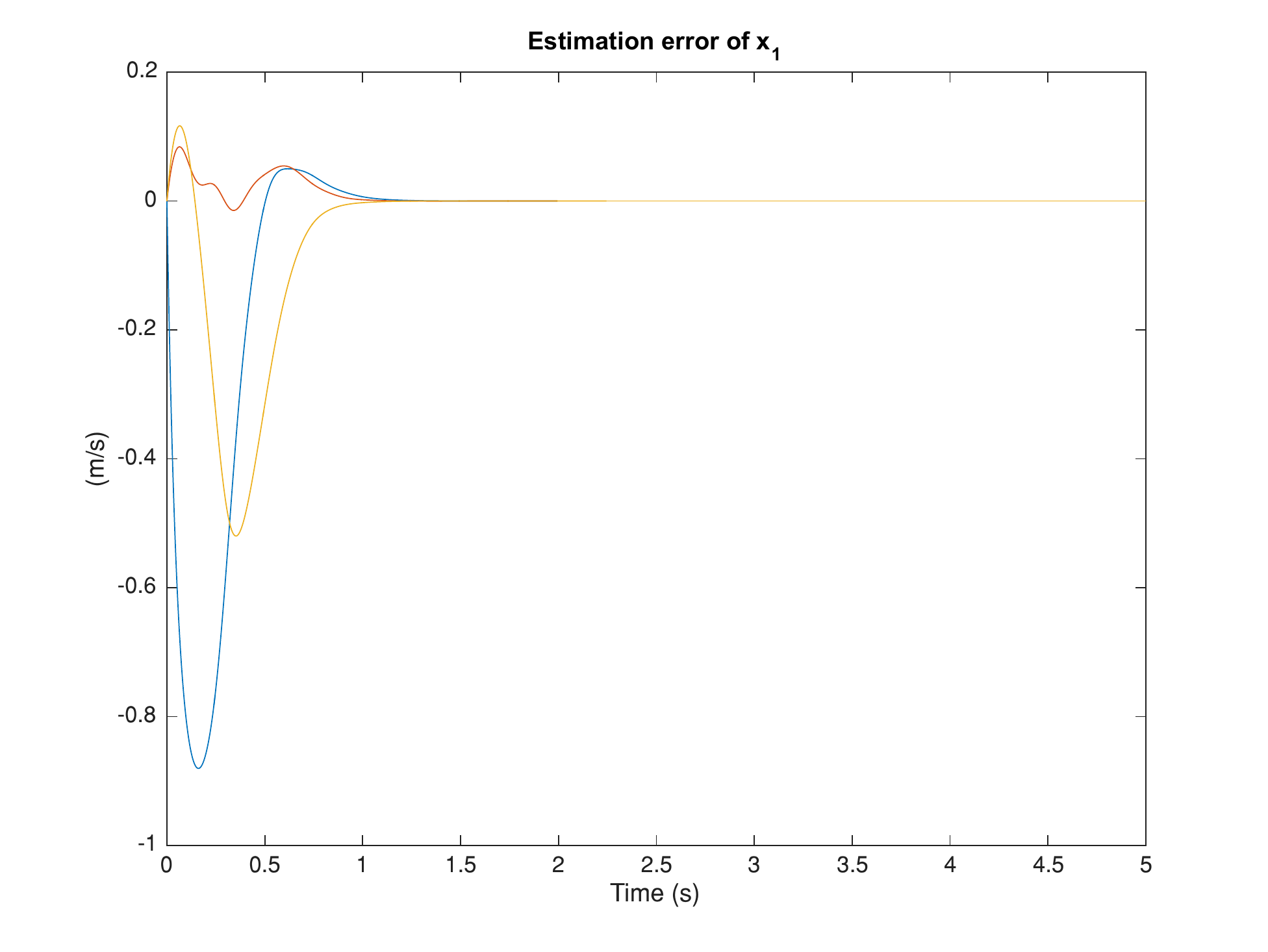} 

\includegraphics[width=1\columnwidth]{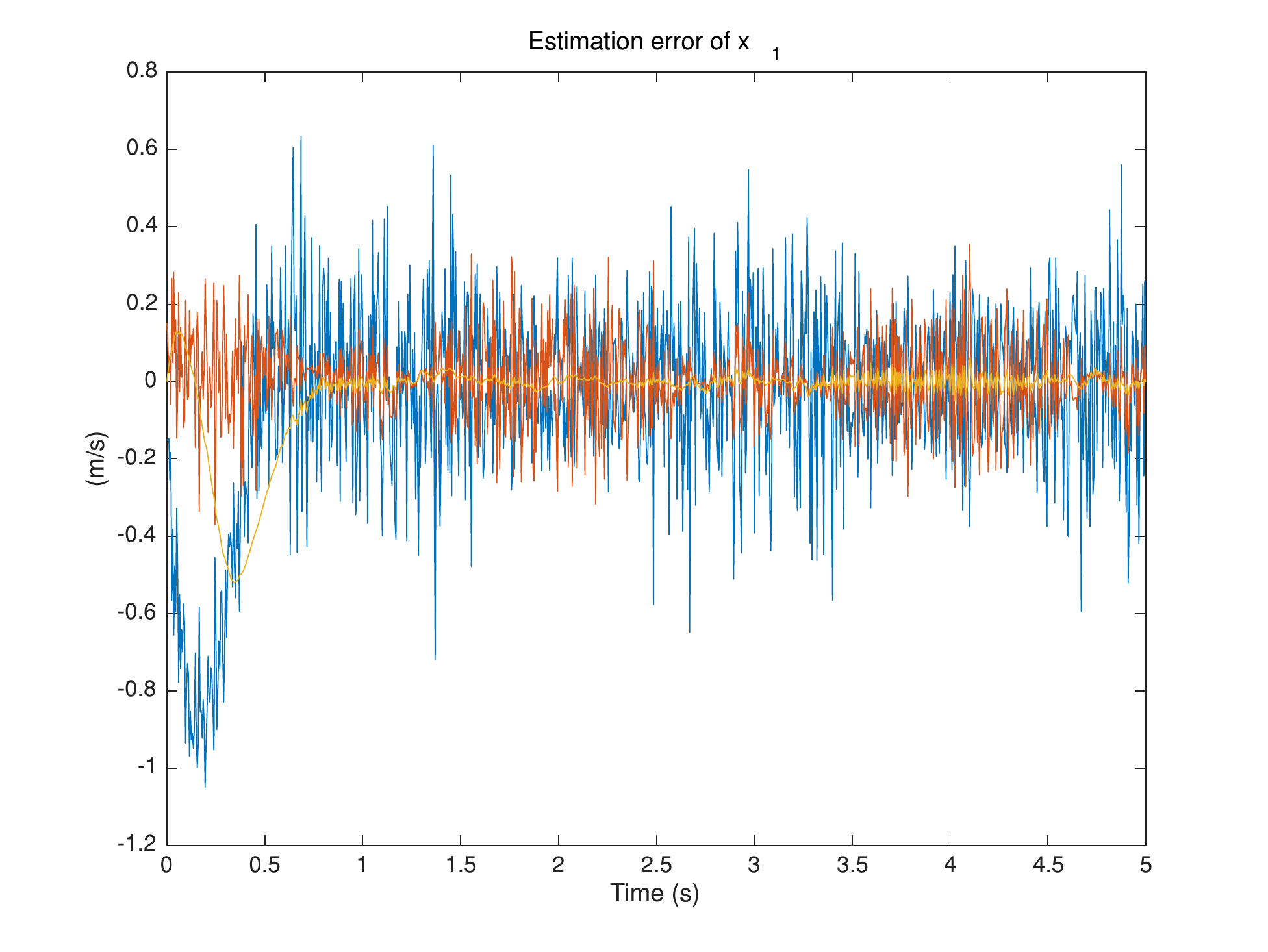}

\caption{Plot showing the estimation error for $x_{1}$. On the top, the case
where there is no noise, and on the bottom the noisy case. The colors
blue, red, orange represent the three components of this vector respectively.
\label{fig:x1tilde}}
\end{figure}

\begin{figure}
\includegraphics[width=0.97\columnwidth]{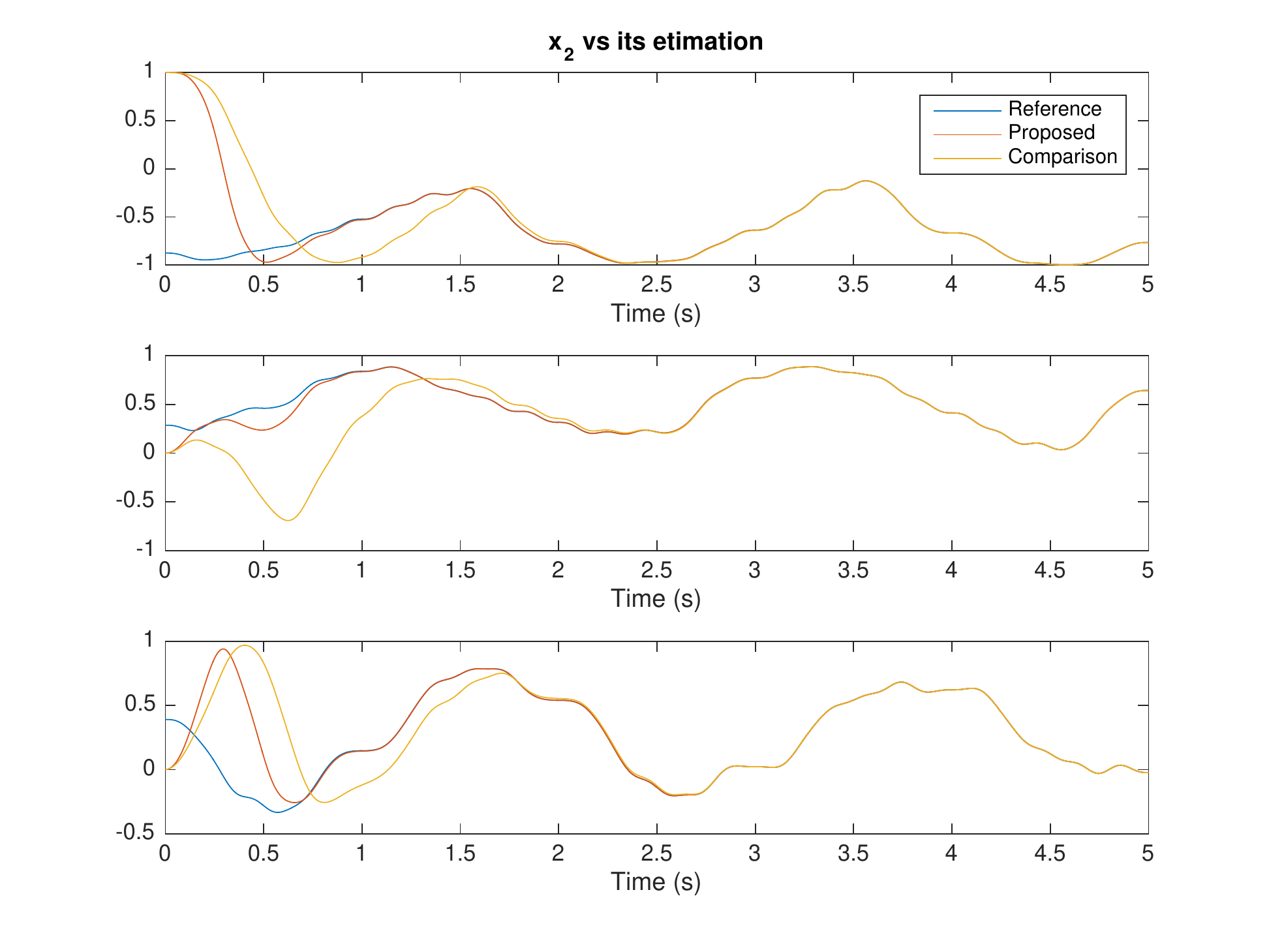}

\includegraphics[width=0.97\columnwidth]{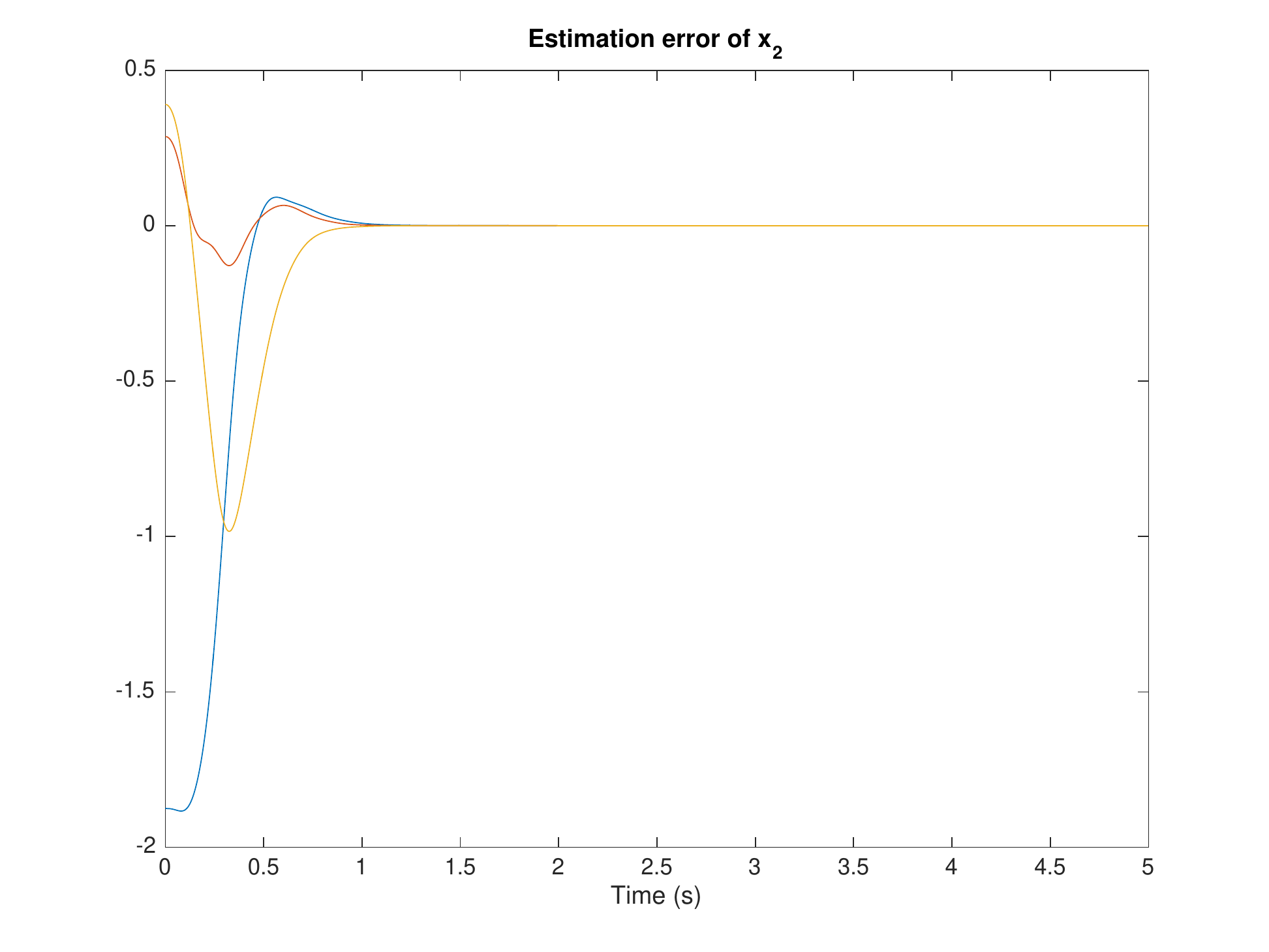}

\caption{Estimation of $x_{2}$ in the case without noise. On the top, three
plots showing a comparison between respective components of $x_{2}$
in blue, its estimation $\hat{x}_{2}$ in red and a comparison with
Hua's observer in yellow. On the bottom we see the evolution of our
estimation error. The colors blue, red, orange represent the three
components of this vector respectively. \label{fig:x2}}
\end{figure}

\begin{figure}
\includegraphics[width=0.97\columnwidth]{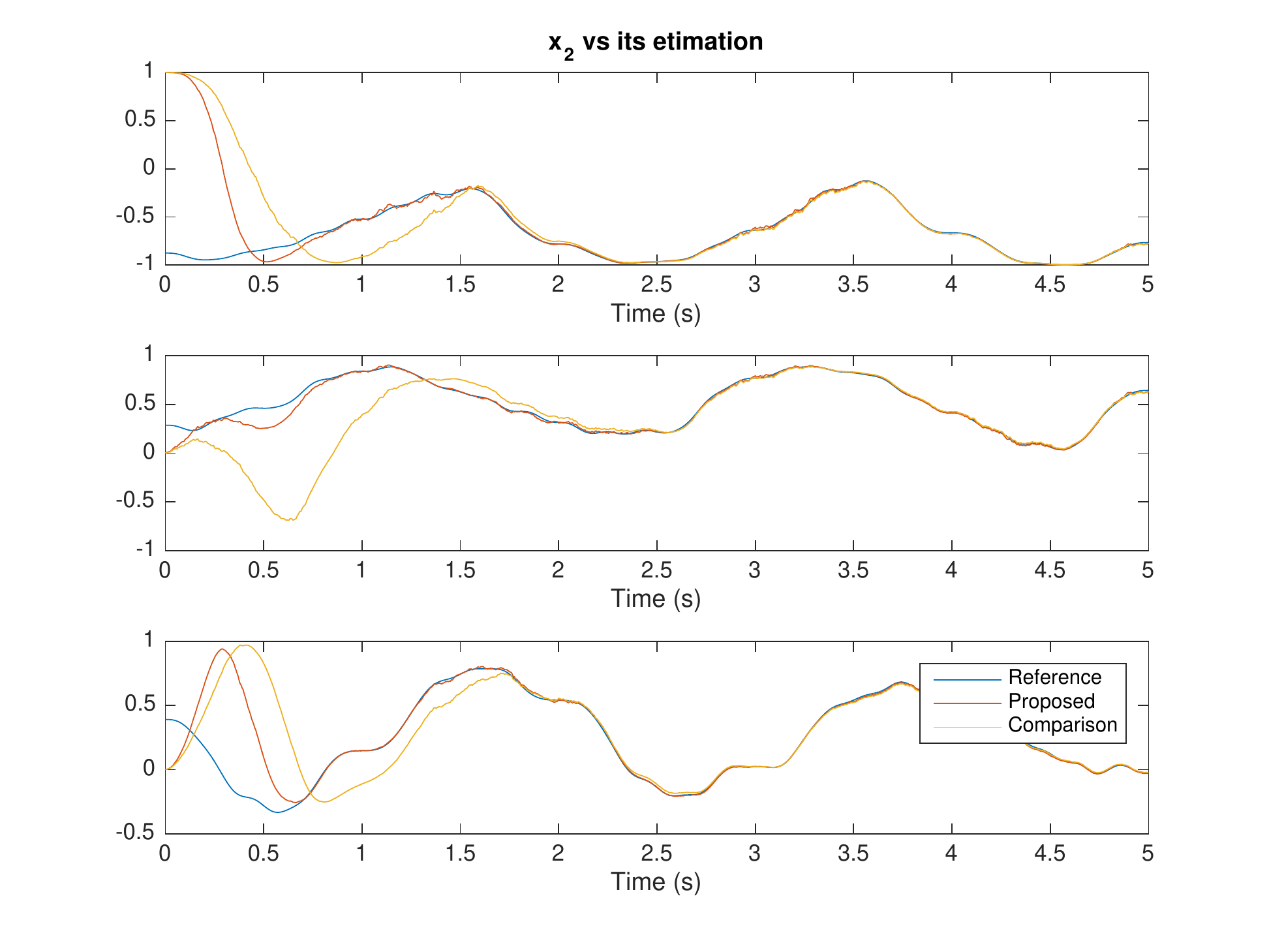}

\includegraphics[width=0.97\columnwidth]{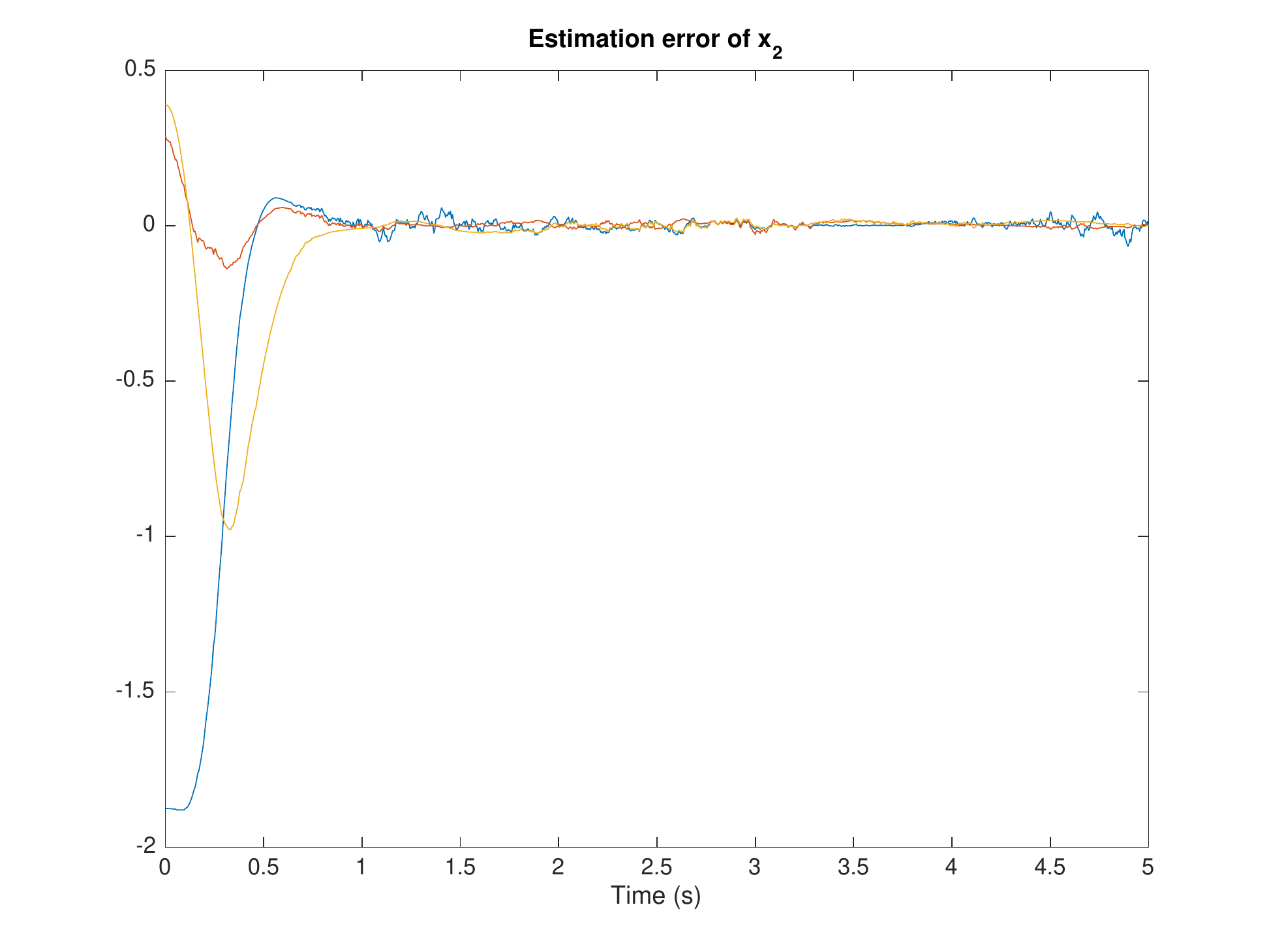}

\caption{Estimation of $x_{2}$ in the noisy case. On the top, three plots
showing a comparison between respective components of $x_{2}$ in
blue and its estimation $\hat{x}_{2}$ in red and a comparison with
Hua's observer in yellow. On the bottom we see the evolution of our
estimation error $\tilde{x}_{2}$. The colors blue, red, orange represent
the three components of this vector respectively.\label{fig:x2noise}}
\end{figure}

\section{Discussion and conclusion\label{sec:Discussion-and-conclusion}}

Attitude estimation is a topic of active research, especially when
IMU signals are used. Accelerometers are at the core of this problem
mainly because their signal contains the value of the gravitational
field in the frame of the sensor. In static cases, this property allows
for algebraic tilt measurements. However, in the dynamic cases, this
measurement is mixed with the linear acceleration in an algebraically
indistinguishable way. In many works the acceleration is considered
negligible compared to gravity field~\cite{Zheng2010}, and is therefore
considered as a noise. Filtering approaches are commonly used to remove
this signal~\cite{1273529}. Accelerometers are also commonly used
together with gyrometers. Gyrometers provide rotation velocities in
the local reference frame. Their signals are commonly merged with
accelerometers using Kalman Filtering~\cite{803999}, but are often
exploited to correct the filtered accelerometer signals using complementary
filtering~\cite{4608934}.

Several other works rely on the presence of additional data to reconstruct
the attitude. For instance, magnetometers~\cite{metni2006attitude}
or vision~\cite{martinelli2012vision} can be used to retrieve redundant
attitude signals allowing to reduce the effect of accelerometer errors.
Finally, a fusion with external measurements such as GPS~\cite{caron2006gps}
or landmark relative position~\cite{vasconcelos2010nonlinear} allow
to better distinguish the linear acceleration from gravitational field
measurements and allows to observe the linear part of the kinematics.

We see through this summary that the translational component of the
motion of the IMU is commonly considered either as a noise that requires
to be deleted or as an independent dynamics which needs to be observed.
However, in the specific case of the pendulum, this linear part of
the kinematics is coupled with the angular motion which explains the
presence of the angular velocity and event angular acceleration in
the signals of the accelerometer (see Equation~(\ref{eq:measure-accelero-1})).
This enables us to use this signal without any need of filtering and
to still be able to reconstruct tilt despite a high level noise level.
The translation-rotation coupling is entirely due to the presence
of the anchor point of pivot. However, in several works addressing
cases similar to pivot link position estimation are still resorting
to classical methods where the IMU is considered as an unconstrained
floating object, even if the reconstructed attitude are merged with
encoder data afterwards~\cite{khandelwal2013estimation}. It is worth
to note that in addition to orientation, the orientation estimation
of a pendulum provides also data on the position of the limbs of the
robot, because of the pivot constraint. This relationship allows also
to design position controllers on the base of attitude estimators,
similarly to hand position compensation presented in~\cite{benallegue2014humanoids}.

Only few works dealt with attitude estimation taking into account
the pivot constraints. One example is the tilt estimation for rigid
pendulum around the pivot using multiple accelerometers \cite{Trimpe2010}.
This observer was used especially for balancing the reaction wheel
cube on edges and corners \cite{Gajamohan2013}. In addition to the
requirement of multiple accelerometers at different locations is only
limited to rigid pendulum cases. Another work from legged robotics
community considers also contact information~\cite{Bloesch-RSS-12}.
This estimator considers the case of multiple contacts and uses an
extended Kalman Filter. The contact information is introduced in the
model kinematics but only at the prediction step rather than as a
constraint. Their model is intended to take into account the cases
of contact slippage, but this variable is not observable using inertial
sensors. Another work uses also extended Kalman Filtering for a humanoid
robot having flexible contacts with the environment~\cite{benallegue2014humanoids}.
The contact information was introduced as pseudo measurements in order
to allow the pivot constraint to be slightly violated. This observer
was extended to take into account the dynamical model of the flexibility~\cite{mifsud:hal-01142399}.
However the use of extended Kalman filtering only provides the guarantee
of optimality around the linearized dynamics around the predicted
state and gives no proof of convergence.

To our knowledge, our estimator is the only one providing almost globally
convergent tilt estimation for non-rigid pendulum system. In fact
the only work we know which can be adapted to these cases is the observer
of\emph{ }\cite{Hua2016automatica}\emph{ }where velocity data are
used to allow tilt estimation. However, this estimator is more complex
than the observer we propose and its convergence properties are weaker,
since no minimal set is provided to guarantee the exponential convergence.
We could see also that their gain condition for stability lead the
system to converge slower than\emph{ }our estimator\emph{. }

With this kind of estimators the only orientation data missing is
the orientation around the gravitational field direction, or yaw angle.
This orientation is proven to be out of reach of this measurement
system. Therefore, the involvement of other sensors such as magnetometers
are necessary to obtain this estimate. The addition of this kind of
sensors is the topic of a possible improvement of the presented method. 

Finally, the introduction of a model for the dynamics of the pivot
could also increase the quality of the observation, specifically by
creating coupling between the measurement data of the IMU and other
values which are non-observable otherwise. These values include yaw
angle without needing additional data, but may go to the estimation
of contact forces with the environment~\cite{mifsud:hal-01142399}.
This is also the topic of next developments regarding this kind of
systems.

\bibliographystyle{plain}
\bibliography{biblio}

\end{document}